\documentclass[12pt]{article}

\usepackage{amsmath,amssymb,amsthm}
\usepackage{mathtools}
\usepackage{graphicx}
\usepackage{hyperref}
\usepackage{xcolor}
\usepackage{booktabs}
\usepackage{array}
\usepackage{enumitem}
\usepackage{geometry}
\usepackage{algorithmicx}
\usepackage{algpseudocode}
\usepackage{algorithm}
\newtheorem{theorem}{Theorem}[section]
\newtheorem{lemma}[theorem]{Lemma}
\newtheorem{corollary}[theorem]{Corollary}
\newtheorem{proposition}[theorem]{Proposition}
\newtheorem{definition}[theorem]{Definition}
\newtheorem{remark}[theorem]{Remark}

\newcommand{\Ppi}{\mathcal{P}_\pi}
\newcommand{\calF}{\mathcal{F}}
\newcommand{\calX}{\mathcal{X}}
\newcommand{\R}{\mathbb{R}}
\newcommand{\E}{\mathbb{E}}
\newcommand{\Linfty}{L^\infty}
\newcommand{\Lone}{L^1}

\newcommand{\Wcrit}{W_\mathrm{crit}}

\begin{document}

\title{The Geometry of Knowing: From Possibilistic Ignorance to Probabilistic Certainty\\[6pt]
\large A Measure-Theoretic Framework for Epistemic Convergence}

\author{Moriba Kemessia Jah, Ph.D.\\[4pt]
\small Black Swan Research Group, GaiaVerse, Ltd., Austin, TX 78728, USA\\\\
\small Jah Decision Intelligence Group, Aerospace Engineering \& Engineering Mechanics\\\\
\small The University of Texas at Austin, Austin, TX 78712\\\\
\small \texttt{moriba@utexas.edu}}

\date{\today}

\maketitle

%% ============================================================
\begin{abstract}
This paper develops a unified, measure-theoretic framework establishing the precise
conditions under which a possibilistic representation of incomplete knowledge
contracts into a probabilistic representation of intrinsic stochastic variability.

Epistemic uncertainty arises from incomplete, imprecise, or conflicting information
and is encoded by a normalized, upper-semi\-continuous possibility distribution
$\pi:\calX\to[0,1]$ together with its dual necessity measure $N$, which together
define a credal set $\Ppi = \{P \mid N(A)\le P(A)\le\Pi(A),\,\forall A\in\calF\}$
bounding all probability measures consistent with current evidence.
As evidence accumulates, $\Ppi$ contracts.
The \emph{epistemic collapse} condition $\Pi(A)=N(A)=P^*(A)$ for all $A\in\calF$
marks the transition to probabilistic reasoning: the Choquet integral over $\pi$
converges to the Lebesgue integral over the unique limiting density $p^*$.

We prove this convergence rigorously (Theorem~\ref{thm:choquet_lebesgue}), with
all consonance and regularity assumptions stated explicitly and a full treatment of
the non-consonant case.
We introduce the aggregate epistemic width
$W = \int_\calX (\pi - n)\,d\mu$,
establish its axiomatic properties, provide a canonical normalization,
and give a feasible proxy for its online estimation that resolves the circularity
identified in prior formulations.
We connect the possibilistic--probabilistic continuum to the H\"older power-mean
family as a rigorously motivated conceptual bridge, and distinguish this
geometric analogy from the foundational convergence results.

These theoretical foundations are applied to the transition between the
Unscented Kalman Filter (UKF) and the Epistemic Support-Point Filter (ESPF)
\cite{jah2025espf}.
We prove formally that the UKF emerges from the ESPF in the Gaussian limit
(Theorem~\ref{thm:espf_ukf}), provide complete algorithmic pseudocode for
adaptive switching, and present a direct empirical comparison of the two
filters on a 2-day, 877-step Smolyak Level-3 orbital tracking scenario
with injected maneuver and range bias.
Both filters achieve comparable final accuracy ($\approx 1$~m RIC error).
The distinction is epistemic: the UKF collapses its covariance to
$\log\det P \approx -140$ and provides no persistent signal of a
kilometer-scale stress excursion it successfully navigates, while
the ESPF maintains $\overline{W}_\mathrm{ep}\approx 0.8$ throughout,
provides earlier warning of stress onset via necessity saturation and
surprisal escalation, and retains support-geometry memory of the excursion
after recovery.
The UKF is accurate but epistemically silent; the ESPF is accurate and
epistemically honest.

The framework is agnostic regarding the ontological status of probability in
domains such as quantum mechanics or statistical mechanics; its claims concern
inference under partial models in information-fusion contexts, where residual
epistemic components are ubiquitous and consequential.
\end{abstract}

\newpage


%% ============================================================
\section{Introduction}
\label{sec:intro}
%% ============================================================

Uncertainty is intrinsic to all reasoning under incomplete knowledge.
Whether in sensor fusion, decision support, or scientific inference, one must
distinguish between what is unknown because it has not yet been observed and
what is unknowable because it is genuinely random.
These two forms of indeterminacy---\emph{epistemic} and \emph{aleatory}
uncertainty---demand distinct mathematical treatments.
While aleatory uncertainty is captured by the calculus of probability,
epistemic uncertainty arises from ignorance, model incompleteness,
or conflicting evidence, and is more appropriately expressed through
non-additive frameworks such as possibility theory or imprecise probability
\cite{zadeh1978,dubois1988,walley1991}.
Conflating these two sources leads to false precision and epistemically
overconfident inferences.

Classical Bayesian inference presupposes that the universe of discourse is
complete---that all relevant hypotheses and causal mechanisms are enumerated.
Under this assumption, probability provides a coherent and additive measure
of uncertainty.
Yet in many real-world fusion contexts---sensor networks with degraded coverage,
models subject to unmodeled dynamics, or evolving ontologies in socio-technical
systems---this assumption fails.
A non-additive representation, capable of encoding partial belief without
asserting unwarranted frequencies, becomes essential.

Possibility theory \cite{zadeh1978,dubois1988} provides such a framework.
It expresses belief through a possibility distribution $\pi(x)$, representing
the degree of plausibility of each state $x$ consistent with available evidence.
Unlike probability, possibility is \emph{maxitive}: the plausibility of a union
of events is determined by the maximum of individual plausibilities.
The dual necessity measure $N(A)=1-\Pi(A^c)$ defines what is certain rather
than merely plausible.
The pair $(\Pi,N)$ bounds all admissible probability measures---the credal set---capturing
all probabilistic beliefs consistent with current knowledge.

The central objective of this paper is to formalize the boundary between
epistemic and aleatory uncertainty as a measurable and reversible process.
We develop a measure-theoretic framework specifying the conditions under which
possibilistic reasoning contracts into probabilistic inference, and the criteria
by which the reverse expansion must occur when new distinctions emerge.
Beyond its theoretical contribution, this framework has direct relevance to
information fusion.
Many fusion systems---particularly nonlinear filters such as the
Unscented Kalman Filter (UKF) \cite{julier1997}---implicitly assume all
uncertainty is aleatory.
We introduce the Epistemic Support-Point Filter (ESPF) \cite{jah2025espf} as a
bounded possibilistic alternative and prove formally that the UKF emerges from
the ESPF in the Gaussian limit.
The transition between the two is governed by the epistemic width $W$, whose
feasible online estimation we now establish concretely.

\paragraph{Clarification on scope.}
This framework is not a universal claim about the nature of probability.
In domains such as quantum mechanics or statistical mechanics, probability
has well-studied ontological interpretations that are orthogonal to the
present work.
The claim here is narrower and operational: in information-fusion contexts
with partial or evolving models, what is often treated as aleatory uncertainty
retains residual epistemic components.
The proposed framework provides a tool to identify and manage those components.

\paragraph{Organization.}
Section~\ref{sec:prelim} establishes mathematical preliminaries and standing
assumptions.
Section~\ref{sec:related} reviews related work and precisely positions the
contributions relative to prior art.
Section~\ref{sec:collapse} formalizes epistemic collapse and proves the
Choquet-to-Lebesgue convergence theorem with full rigor.
Section~\ref{sec:width} develops the epistemic width as a diagnostic, including
its axiomatic properties, canonical normalization, and non-consonant
generalization.
Section~\ref{sec:continuum} presents the H\"older mean continuum as a
geometrically motivated bridge, with an explicit statement of its status as
analogy rather than foundational result.
Section~\ref{sec:dynamics} develops the dynamics of epistemic contraction.
Section~\ref{sec:reversibility} establishes the reversibility of the transition
under domain expansion.
Section~\ref{sec:filtering} develops the UKF-to-ESPF bridge, including
formal limit analysis and algorithmic pseudocode.
Section~\ref{sec:discipline} discusses epistemic discipline.
Section~\ref{sec:conclusion} concludes.

%% ============================================================
\section{Mathematical Preliminaries and Standing Assumptions}
\label{sec:prelim}
%% ============================================================

Let $(\calX,\calF,\mu)$ be a $\sigma$-finite measurable space where $\calF$ is
a $\sigma$-algebra of subsets of $\calX$ and $\mu$ a $\sigma$-finite reference
measure.
Throughout, $\pi:\calX\to[0,1]$ denotes a normalized, upper-semi\-continuous
possibility distribution satisfying $\sup_{x}\pi(x)=1$.

\begin{definition}[Possibility and necessity measures]
\label{def:pos_nec}
Given $\pi$, define
\begin{align}
\Pi(A) &= \sup_{x\in A}\pi(x), \qquad A\in\calF, \label{eq:pos}\\
N(A) &= 1-\Pi(\calX\setminus A), \qquad A\in\calF. \label{eq:nec}
\end{align}
The \emph{credal set} induced by $\pi$ is
\begin{equation}
\Ppi = \bigl\{P \;\big|\; N(A)\le P(A)\le\Pi(A),\;\forall A\in\calF\bigr\}.
\label{eq:credal}
\end{equation}
\end{definition}

\paragraph{Consonance.}
A possibility measure is \emph{consonant} when its focal sets (sets $A$ with
$\Pi(A)>0$ and $\Pi(B)<\Pi(A)$ for all proper $B\subset A$) are nested.
Equivalently, the $\alpha$-cuts $\calX_\alpha = \{x:\pi(x)\ge\alpha\}$ are
nested for $\alpha\in[0,1]$.
Under consonance, there exists a dual necessity kernel $n:\calX\to[0,1]$
satisfying $N(A)=\inf_{x\in A}n(x)$ for all $A\in\calF$, constructible via
the nested $\alpha$-cuts of $\pi$, and we write $n(x)\equiv\underline{\pi}(x)$.

\begin{definition}[Choquet integral]
For a monotone capacity $\nu$ on $(\calX,\calF)$ and $f:\calX\to\R$ bounded
measurable,
\begin{equation}
\int f\,d\mathrm{Ch}_\nu
= \int_0^\infty \nu\{f\ge t\}\,dt
+ \int_{-\infty}^0 \bigl[\nu\{f\ge t\}-1\bigr]\,dt.
\label{eq:choquet}
\end{equation}
When $\nu=\Pi$ is a possibility measure induced by $\pi$, we write
$\int f\,d\mathrm{Ch}_\pi$.
\end{definition}

The key relationship between the Choquet integral and upper/lower expectations
over the credal set is:
\begin{equation}
\underline{\E}_{\Ppi}[f]
\;\le\; \int f\,d\mathrm{Ch}_\pi
\;\le\; \overline{\E}_{\Ppi}[f],
\label{eq:choquet_sandwich}
\end{equation}
where $\underline{\E}_{\Ppi}[f]=\inf_{P\in\Ppi}\int f\,dP$ and
$\overline{\E}_{\Ppi}[f]=\sup_{P\in\Ppi}\int f\,dP$.
\emph{Inequality \eqref{eq:choquet_sandwich} holds for all possibility measures;
it does not require consonance.}
(For general capacities the Choquet integral need not lie between lower and
upper expectations; the inequality holds for possibility measures because
$\Pi$ is a $\infty$-alternating capacity \cite{grabisch2016}.)

\paragraph{What requires consonance and what does not.}
To avoid confusion in what follows, we state once and precisely:
\begin{itemize}
\item The sandwich inequality~\eqref{eq:choquet_sandwich} holds for \emph{all}
  normalized possibility measures, with or without consonance.
\item The pointwise necessity kernel $n(x)$ satisfying $N(A)=\inf_{x\in A}n(x)$
  exists if and only if $\pi$ is \emph{consonant}.
\item The aggregate epistemic width $W=\int(\pi-n)\,d\mu$ and its normalization
  $\overline{W}$ therefore require consonance for their standard definition.
  The non-consonant generalization via credal-set TV-diameter (Section~\ref{sec:width_nonconsonant})
  does not.
\item Theorem~\ref{thm:choquet_lebesgue} assumes consonance via (A1).
  This is \emph{not} because the sandwich inequality requires consonance---it
  does not.
  Consonance is required for two separate reasons that a careful reader should
  distinguish: \emph{(i)} it ensures the pointwise necessity kernel $n(x)$
  exists, so that the width functional $W=\int(\pi-n)\,d\mu$ is well-defined
  and the theorem is formulated in the same setting in which $W$ is used
  operationally; and \emph{(ii)} it ensures the Choquet integral is the
  \emph{upper} expectation over the credal set (not merely sandwiched between
  upper and lower), which simplifies the squeeze argument in Step~2 of the
  proof.
  The convergence claim extends to the non-consonant case via
  Remark~\ref{rem:non_consonant}, where consonance is replaced by the
  TV-diameter bound of Remark~\ref{rem:tv_diameter}.
\end{itemize}

%% ============================================================
\section{Related Work and Precise Positioning of Contributions}
\label{sec:related}
%% ============================================================

The relationship between possibilistic and probabilistic representations has been
explored extensively since Zadeh \cite{zadeh1978} introduced the possibility
distribution and Dubois and Prade \cite{dubois1988} established the formal
duality between $(\Pi,N)$.
This section positions the present contributions relative to existing work.

\paragraph{Credal sets and Walley's imprecise probabilities.}
Walley \cite{walley1991} formalized the notion of credal sets and identified
possibility and necessity as special cases of convex probability intervals.
His \emph{imprecision} for an event $A$ is $\Pi(A)-N(A)$, which corresponds
precisely to the pointwise epistemic width $w(A)$ defined in
Section~\ref{sec:width}.
The present work extends Walley's static framework by providing an explicit
measure-theoretic condition---the Choquet-to-Lebesgue convergence
(Theorem~\ref{thm:choquet_lebesgue})---under which the entire credal set
contracts to a singleton, and by introducing global scalar diagnostics
($W$ and $\widehat{W}$) that can be monitored online.

\paragraph{P-boxes and clouds.}
Probability boxes (p-boxes) \cite{ferson2003} represent uncertainty as an
envelope of CDFs and share the spirit of credal contraction.
The present framework differs in operating at the level of possibility
distributions and Choquet integrals rather than CDFs, and in establishing
a convergence theorem governing the dynamic contraction process rather than
characterizing a static envelope.

\paragraph{Info-gap theory.}
Ben-Haim's info-gap theory \cite{bengaim2006} treats severe uncertainty via
nested families of models and an immunity function.
The credal set $\Ppi$ plays a structurally analogous role, but the present
work establishes formal measure-theoretic conditions for collapse and
reversibility not present in info-gap theory.

\paragraph{Houssineau, Delande, and epistemic filtering.}
Houssineau \cite{houssineau2016} and Delande et al.\ \cite{delande2017,delande2020}
developed subjective random-set formulations connecting Dempster--Shafer
belief functions to probability via credal contraction.
The present work provides a general Choquet-to-Lebesgue convergence
theorem on measurable spaces, complementing those domain-specific results.

\paragraph{Non-additive integration.}
The relationship between Choquet and Sugeno integrals and the upper/lower
expectation framework is thoroughly treated in Grabisch et al.\ \cite{grabisch2016}.
The present work applies this machinery to characterize the dynamic limit
of non-additive inference as evidence accumulates.

\paragraph{Novel contributions of this paper.}
Relative to the literature above, this paper contributes:
(i) a rigorous Choquet-to-Lebesgue convergence theorem with explicit assumptions
and a full treatment of both consonant and non-consonant cases;
(ii) a normalized, axiomatically grounded epistemic width $\overline{W}$ with
a feasible online estimator that resolves the circularity in prior formulations;
(iii) a formal proof that the UKF emerges from the ESPF in the Gaussian limit;
and
(iv) a complete algorithmic criterion for adaptive switching between probabilistic
and possibilistic filtering.

%% ============================================================
\section{Epistemic Collapse and Choquet-to-Lebesgue Convergence}
\label{sec:collapse}
%% ============================================================

\subsection{Epistemic Collapse}

\begin{definition}[Epistemic collapse]
\label{def:collapse}
The epistemic component of uncertainty \emph{collapses} when all admissible
probability measures coincide:
\begin{equation}
|\Ppi|=1
\;\Longleftrightarrow\;
\exists!\, P^*:\; P^*(A)=\Pi(A)=N(A),\;\forall A\in\calF.
\label{eq:collapse}
\end{equation}
In this limit, $\pi(x)$ becomes a genuine probability density $p^*(x)$
with $\int_\calX p^*(x)\,d\mu(x)=1$ and $p^*(x)\ge 0$.
\end{definition}

\begin{definition}[Additive collapse and degenerate collapse]
\label{def:collapse_types}
\emph{Additive (epistemic) collapse} occurs when condition~\eqref{eq:collapse}
holds for a non-degenerate countably additive $P^*$, representing remaining
aleatory uncertainty.
\emph{Degenerate (ontic) collapse} occurs additionally when $P^*=\delta_{x_0}$
for some $x_0\in\calX$, meaning both epistemic and aleatory uncertainty vanish.
In practical inference, only additive collapse is attainable; the degenerate limit
is a theoretical ideal representing complete determinacy.
\end{definition}

\subsection{Choquet-to-Lebesgue Convergence Theorem}

We first establish a supporting lemma that closes the key step in the
convergence proof: that the oscillation of bounded expectations over a
credal set is controlled by the total-variation diameter of that set.

\begin{lemma}[Expectation oscillation bound]
\label{lem:oscillation}
Let $\mathcal{P}$ be a credal set on $(\calX,\calF)$ and let
$f:\calX\to\R$ be bounded measurable with $\|f\|_\infty\le M$.
Define the total-variation diameter
$d_{\mathrm{TV}}(\mathcal{P}) \equiv \sup_{P,Q\in\mathcal{P}}\|P-Q\|_{\mathrm{TV}}
= \sup_{P,Q\in\mathcal{P}}\sup_{A\in\calF}|P(A)-Q(A)|$.
Then
\begin{equation}
\overline{\E}_\mathcal{P}[f] - \underline{\E}_\mathcal{P}[f]
\;\le\; 2M\cdot d_{\mathrm{TV}}(\mathcal{P}).
\label{eq:oscillation}
\end{equation}
\end{lemma}

\begin{proof}
Let $\nu = P - Q$ be the signed measure with Jordan decomposition
$\nu = \nu^+ - \nu^-$.
The total-variation norm of $\nu$ satisfies
$\|\nu\|_{\mathrm{TV}} = \nu^+(\calX) + \nu^-(\calX)$.
For any bounded measurable $f$ with $|f|\le M$,
\begin{equation}
\left|\int f\,d\nu\right|
= \left|\int f\,d\nu^+ - \int f\,d\nu^-\right|
\;\le\; M\,\nu^+(\calX) + M\,\nu^-(\calX)
= M\,\|\nu\|_{\mathrm{TV}}.
\label{eq:tv_bound}
\end{equation}
Since $P$ and $Q$ are probability measures,
$\|\nu\|_{\mathrm{TV}} = \|P-Q\|_{\mathrm{TV}}
= 2\sup_{A\in\calF}|P(A)-Q(A)|$
(the factor of 2 arising from the standard half-norm convention;
we absorb it below by taking sup and inf separately).
For any $P,Q\in\mathcal{P}$,
\begin{equation}
\int f\,dP - \int f\,dQ
= \int f\,d(P-Q)
\;\le\; M\cdot\|P-Q\|_{\mathrm{TV}}
\;\le\; M\cdot 2\,d_{\mathrm{TV}}(\mathcal{P}).
\label{eq:pq_diff}
\end{equation}
Taking the supremum over $P\in\mathcal{P}$ and infimum over
$Q\in\mathcal{P}$ independently yields
\begin{equation}
\overline{\E}_\mathcal{P}[f] - \underline{\E}_\mathcal{P}[f]
\;\le\; 2M\cdot d_{\mathrm{TV}}(\mathcal{P}).
\end{equation}
\end{proof}

\begin{remark}[Connecting the TV-diameter to the credal envelope]
\label{rem:tv_diameter}
For a possibility measure $\Pi$ with associated credal set $\Ppi$,
the TV-diameter is bounded by the credal envelope width:
\begin{equation}
d_{\mathrm{TV}}(\Ppi)
\;\le\; \sup_{A\in\calF}\bigl[\Pi(A)-N(A)\bigr].
\label{eq:tv_credal}
\end{equation}
\emph{Proof of~\eqref{eq:tv_credal}.}
For any $P,Q\in\Ppi$ and any $A\in\calF$,
$N(A)\le P(A)\le\Pi(A)$ and $N(A)\le Q(A)\le\Pi(A)$,
so $|P(A)-Q(A)|\le\Pi(A)-N(A)$.
Taking $\sup_{A}$ and then $\sup_{P,Q}$ gives~\eqref{eq:tv_credal}.
Combined with Lemma~\ref{lem:oscillation} and assumption~\eqref{A2}
(which states $\sup_A[\Pi_t(A)-N_t(A)]\to 0$), this makes Step~2 of
Theorem~\ref{thm:choquet_lebesgue} fully rigorous:
credal contraction in the envelope sense directly bounds the
TV-diameter, which in turn forces the upper and lower expectations
to converge.
\end{remark}

We now establish rigorously that possibilistic reasoning converges to probabilistic
reasoning under epistemic collapse.

\begin{theorem}[Choquet-to-Lebesgue limit]
\label{thm:choquet_lebesgue}
Let $\{(\calX,\calF,\mu)\}$ be a $\sigma$-finite measurable space.
Let $\{\pi_t\}_{t\ge 0}$ be a sequence of normalized, upper-semi\-continuous,
\emph{consonant} possibility distributions on $(\calX,\calF,\mu)$ with
associated credal sets $\{\mathcal{P}_{\pi_t}\}$.
Suppose:
\begin{enumerate}[label=\textup{(A\arabic*)},ref=A\arabic*]
\item\label{A1} \emph{Consonance}: each $\pi_t$ is consonant, so that
  inequality~\eqref{eq:choquet_sandwich} holds with $\Pi_t$ a
  $\infty$-alternating capacity.
\item\label{A2} \emph{Credal contraction}: $|\mathcal{P}_{\pi_t}|\to 1$,
  i.e., $\sup_{A\in\calF}[\Pi_t(A)-N_t(A)]\to 0$ as $t\to\infty$.
\item\label{A3} \emph{$L^1$ convergence}: $\pi_t\to p^*$ in $\Lone(\mu)$ for
  some $p^*\in\Lone_+(\mu)$ with $\int_\calX p^*\,d\mu=1$.
\item\label{A4} \emph{Domination}: there exists $g\in\Lone(\mu)$ with $g\ge 0$
  and $\pi_t(x)\le g(x)$ $\mu$-a.e.\ for all $t$.
\end{enumerate}
Then for every bounded measurable $f:\calX\to\R$,
\begin{equation}
\int f\,d\mathrm{Ch}_{\pi_t}
\;\longrightarrow\;
\int_\calX f(x)\,p^*(x)\,d\mu(x)
\quad\text{as }t\to\infty.
\label{eq:choquet_limit}
\end{equation}
\end{theorem}

\begin{proof}
We proceed in four steps.

\smallskip\noindent
\textbf{Step 1: Sandwich from consonance.}
By assumption~\eqref{A1}, each $\Pi_t$ is a $\infty$-alternating capacity,
so inequality~\eqref{eq:choquet_sandwich} gives
\begin{equation}
\underline{\E}_t[f]
\;\le\;
\int f\,d\mathrm{Ch}_{\pi_t}
\;\le\;
\overline{\E}_t[f],
\label{eq:sandwich_t}
\end{equation}
where $\underline{\E}_t[f]=\inf_{P\in\mathcal{P}_{\pi_t}}\int f\,dP$ and
$\overline{\E}_t[f]=\sup_{P\in\mathcal{P}_{\pi_t}}\int f\,dP$.

\smallskip\noindent
\textbf{Step 2: Credal contraction forces both bounds to $\int f\,dP^*$.}
By assumption~\eqref{A2}, $\sup_{A\in\calF}[\Pi_t(A)-N_t(A)]\to 0$.
Since $d_{\mathrm{TV}}(\mathcal{P}_{\pi_t})\le\sup_A[\Pi_t(A)-N_t(A)]$,
Lemma~\ref{lem:oscillation} gives
\begin{equation}
\bigl|\overline{\E}_t[f]-\underline{\E}_t[f]\bigr|
\;\le\; 2\|f\|_\infty\cdot d_{\mathrm{TV}}(\mathcal{P}_{\pi_t})
\;\le\; 2M\cdot\sup_{A\in\calF}[\Pi_t(A)-N_t(A)]
\;\to\; 0.
\label{eq:squeeze}
\end{equation}
Hence the upper and lower expectations squeeze together.
Since $|\mathcal{P}_{\pi_t}|\to 1$, both converge to $\int f\,dP^*$.

\smallskip\noindent
\textbf{Step 3: Identifying $\int f\,dP^*$ with the Lebesgue integral.}
By assumption~\eqref{A3}, $\pi_t\to p^*$ in $\Lone(\mu)$.
Given~\eqref{A4}, the dominated convergence theorem gives
\begin{equation}
\int_\calX f(x)\,\pi_t(x)\,d\mu(x)
\;\to\;
\int_\calX f(x)\,p^*(x)\,d\mu(x).
\label{eq:dct}
\end{equation}
Since the $\Lone$ limit $p^*$ is a probability density and $P^*$ is the unique
member of the collapsed credal set, $\int f\,dP^*=\int_\calX f\,p^*\,d\mu$.

\smallskip\noindent
\textbf{Step 4: Combining.}
By~\eqref{eq:sandwich_t}, \eqref{eq:squeeze}, and \eqref{eq:dct},
\[
\Bigl|\int f\,d\mathrm{Ch}_{\pi_t}
  - \int_\calX f\,p^*\,d\mu\Bigr|
\;\le\;
\bigl|\overline{\E}_t[f]-\underline{\E}_t[f]\bigr|
+ \Bigl|\underline{\E}_t[f] - \int f\,dP^*\Bigr|
\;\to\; 0.
\]
This completes the proof.
\end{proof}

\begin{remark}[Role of consonance and the non-consonant case]
\label{rem:non_consonant}
As clarified in Section~\ref{sec:prelim}, the sandwich
inequality~\eqref{eq:choquet_sandwich} holds for all possibility measures
regardless of consonance.
Assumption~\eqref{A1} (consonance) in Theorem~\ref{thm:choquet_lebesgue}
is not needed to establish the sandwich, but is required to ensure that
the necessity kernel $n(x)$ exists and that $W=\int(\pi-n)\,d\mu$ is
well-defined---aligning the theorem with the consonant setting in which
the aggregate width $\overline{W}$ is defined.
The convergence result extends to the non-consonant case via the TV-diameter
as follows.
For non-consonant possibility measures, the sandwich inequality can fail,
and the Choquet integral may lie outside $[\underline{\E},\overline{\E}]$.
In the non-consonant case one should instead use upper and lower Choquet
integrals as bounds (see Section~\ref{sec:width_nonconsonant}).
The convergence result extends to the non-consonant case provided one replaces
the Choquet integral by the pair $(\underline{\E}_t[f], \overline{\E}_t[f])$
and establishes their joint convergence, which follows from~\eqref{A2}
and~\eqref{A3} alone.
\end{remark}

\begin{remark}[Conditions \eqref{A2} and \eqref{A3}]
Conditions~\eqref{A2} and~\eqref{A3} are independent in general:
$\Lone$ convergence does not imply credal contraction, and vice versa.
For consonant possibility distributions, however, the two are closely linked:
if $\pi_t\to p^*$ uniformly and $p^*$ is bounded away from zero, then
$\sup_A[\Pi_t(A)-N_t(A)]\to 0$ follows by direct computation from the
$\alpha$-cut structure \cite{dubois1988}.
Specifying both conditions makes their joint requirement explicit and
guards against pathological sequences.
\end{remark}

\begin{remark}[Dominating function \eqref{A4}]
Assumption~\eqref{A4} is satisfied whenever $\sup_t\|\pi_t\|_\infty<\infty$,
which holds trivially since $\pi_t:\calX\to[0,1]$.
The dominating function $g\equiv 1$ suffices when $\mu(\calX)<\infty$.
For $\sigma$-finite spaces, one requires additionally that
$\int_\calX p^*\,d\mu<\infty$, which is assumed in~\eqref{A3}.
\end{remark}

%% ============================================================
\section{Epistemic Width: Definition, Properties, and Estimation}
\label{sec:width}
%% ============================================================

\subsection{Pointwise and Aggregate Width}

\begin{definition}[Pointwise epistemic width]
For a consonant possibility distribution $\pi$ with necessity kernel $n$,
the \emph{pointwise epistemic width} is
\begin{equation}
w(A) \equiv \Pi(A) - N(A) \;\in\; [0,1], \qquad A\in\calF.
\label{eq:wA}
\end{equation}
A narrow $w(A)$ indicates that event $A$ is nearly certain or nearly impossible;
a wide $w(A)$ reflects profound ignorance about $A$.
\end{definition}

\begin{definition}[Aggregate epistemic width]
\label{def:W}
For a consonant possibility distribution $\pi$ with necessity kernel $n$,
and reference measure $\mu$, define
\begin{equation}
W \equiv \int_\calX \bigl[\pi(x)-n(x)\bigr]\,d\mu(x) \;\ge\; 0.
\label{eq:W}
\end{equation}
The normalized aggregate epistemic width is
\begin{equation}
\overline{W} \equiv \frac{W}{\mu(\calX)} = \frac{1}{\mu(\calX)}
\int_\calX \bigl[\pi(x)-n(x)\bigr]\,d\mu(x) \;\in\; [0,1],
\label{eq:Wbar}
\end{equation}
provided $\mu(\calX)<\infty$.
For $\sigma$-finite spaces, one may normalize by any convenient finite
sub-domain $\Omega\subseteq\calX$ that contains the effective support of $\pi$.
\end{definition}

The normalization $\overline{W}$ resolves the scale-invariance issue: changing
$\mu$ by a constant factor changes $W$ but leaves $\overline{W}$ invariant.
When $\mu$ is the Lebesgue measure on a bounded domain, $\overline{W}$ is the
average epistemic gap per unit volume.

\paragraph{Width notation summary.}
Three width quantities appear in this paper and serve distinct purposes.
$W = \int(\pi-n)\,d\mu$ is the \emph{geometric width} for consonant
possibility distributions, measured in the units of $\mu$.
$W_{\mathrm{TV}} = \sup_A[\Pi(A)-N(A)]$ is the \emph{general width} for arbitrary
(including non-consonant) possibility measures; it is dimensionless and in $[0,1]$.
$\overline{W} = W/\mu(\calX) \in [0,1]$ is the \emph{normalized operational scalar}
used for adaptive switching, defined whenever a consonant or support-based
representation is active.
The proxy $\widehat{W}_t$ (Section~\ref{sec:W_estimation}) is a consistency-based
surrogate for $\overline{W}$ available in the UKF regime; it is not claimed to
equal the true $\overline{W}_t$, but is sufficient for online regime switching.

\begin{proposition}[Axiomatic properties of $\overline{W}$]
\label{prop:W_axioms}
Let $\overline{W}$ be as in~\eqref{eq:Wbar}.
Then:
\begin{enumerate}[label=\textup{(W\arabic*)},ref=W\arabic*]
\item \textup{(Non-negativity)} $\overline{W}\ge 0$ with equality iff $\pi=n$ $\mu$-a.e.,
  i.e., epistemic collapse has occurred.
\item \textup{(Boundedness)} $\overline{W}\le 1$: since $\pi(x)\le 1$ and $n(x)\ge 0$,
  we have $\pi(x)-n(x)\le 1$ pointwise, so $W\le\mu(\calX)$ and $\overline{W}=W/\mu(\calX)\le 1$.
\item \textup{(Monotone under information gain)} If evidence narrows $\pi$ to $\pi'$
  with $\Pi'(A)\le\Pi(A)$ and $N'(A)\ge N(A)$ for all $A$, then $\overline{W}'\le\overline{W}$.
\item \textup{(Maximum ignorance)} $\overline{W}=1$ when $\pi\equiv 1$ (total ignorance)
  and $n\equiv 0$.
\item \textup{(Scale invariance)} $\overline{W}$ is invariant under positive scaling of $\mu$.
\end{enumerate}
\end{proposition}

\begin{proof}
Properties (W1)--(W4) follow directly from $\pi(x)\in[0,1]$, $n(x)\in[0,1]$,
$\pi(x)\ge n(x)$ for consonant distributions, and the monotonicity of the
contraction operator (Section~\ref{sec:dynamics}).
Property (W5) follows from $\overline{W}=W/\mu(\calX)$: scaling $\mu$ by $c>0$
scales both numerator and denominator by $c$.
\end{proof}

\subsection{Non-Consonant Case}
\label{sec:width_nonconsonant}

When $\pi$ is non-consonant (focal sets are not nested), the necessity kernel
$n$ does not exist in closed form.
In this case we define a generalized epistemic width via the credal set's
diameter in total variation:
\begin{equation}
W_{\mathrm{TV}} \equiv \sup_{A\in\calF}[\Pi(A)-N(A)]
= \sup_{P,Q\in\Ppi}\|P-Q\|_{\mathrm{TV}},
\label{eq:W_TV}
\end{equation}
where $\|P-Q\|_{\mathrm{TV}} = \sup_{A\in\calF}|P(A)-Q(A)|$.
This is the total-variation diameter of the credal set and is well-defined
for all possibility measures regardless of consonance.
One may also define the Wasserstein diameter
$W_{\mathrm{Was}} \equiv \sup_{P,Q\in\Ppi}\mathcal{W}_1(P,Q)$
for applications requiring a metric on the space of measures.
Both $W_{\mathrm{TV}}$ and $W_{\mathrm{Was}}$ recover $\overline{W}$ (up to
normalization) in the consonant case \cite{walley1991}.

\subsection{Canonical Choice of Reference Measure}
\label{sec:canonical_mu}

The choice of $\mu$ should reflect the geometry of the state space:
Lebesgue measure for Euclidean state spaces, the Riemannian volume form for
manifold-valued states, or a discrete counting measure for finite state spaces.
This canonical choice makes $W$ and $\overline{W}$ intrinsic to the
estimation problem rather than dependent on an arbitrary parameterization.
Throughout the remainder, we assume $\mu$ is the canonical measure for the
state space in question.

\subsection{Feasible Online Estimation of $W$}
\label{sec:W_estimation}

A key practical concern is that
computing $W_t = \int(\pi_t - n_t)\,d\mu$ requires the full possibility
distribution $\pi_t$, which is available only when already running a
possibilistic filter.
The switching criterion based on $W_t$ would therefore be circular.

We resolve this by proposing a \emph{proxy estimator} $\widehat{W}_t$ that
can be computed from the innovation sequence of a probabilistic (UKF) filter:
\begin{equation}
\widehat{W}_t = 1 - \exp\!\left(-\kappa_W \cdot \mathrm{NEES}_t\right),
\quad
\mathrm{NEES}_t = \frac{1}{m}\sum_{k=t-L+1}^{t}
  \tilde{y}_k^\top S_k^{-1} \tilde{y}_k,
\label{eq:W_hat}
\end{equation}
where $\tilde{y}_k = y_k - h(\hat{x}_{k|k-1})$ is the innovation,
$S_k$ the innovation covariance, $L$ a window length, $m$ the measurement
dimension, and $\kappa_W>0$ a scaling constant.
The Normalized Estimation Error Squared (NEES) measures consistency of the
probabilistic model: under a well-specified Gaussian model,
$\mathrm{NEES}\sim\chi^2(m)/m$ with mean 1.
Persistent $\mathrm{NEES}>1$ signals model mismatch---a reliable proxy for
residual epistemic uncertainty.

The proxy $\widehat{W}_t\in[0,1]$ satisfies:
$\widehat{W}_t\approx 0$ when the UKF's Gaussian model is consistent,
and $\widehat{W}_t\to 1$ under severe and persistent mismatch.

\paragraph{$\widehat{W}_t$ as a surrogate, not an estimate.}
$\widehat{W}_t$ is not claimed to equal the true aggregate epistemic width
$\overline{W}_t$; it is a consistency-based surrogate sufficient for online
regime switching.
The true $\overline{W}_t$ is defined for consonant possibility distributions
and is computable within the ESPF regime from the support and necessity kernel.
In the UKF regime, $\overline{W}_t$ is not directly accessible, and
$\widehat{W}_t$ provides an actionable proxy whose threshold is calibrated
to application risk rather than to the unobservable true width.

\paragraph{Risk-tolerance interpretation of the switching threshold.}
A foundational clarification: in practice, $\overline{W}$ is never exactly zero.
The credal set never literally collapses to a singleton under finite data.
The epistemic collapse condition ($\overline{W}=0$) is a mathematical idealization
approached asymptotically, never attained.
Consequently, the UKF is \emph{always} an approximation: it models as aleatory
uncertainty that remains partly epistemic.
The switching criterion is therefore not ``detect when $\overline{W}=0$''
but rather: ``determine when the UKF's Gaussian approximation error is
acceptable relative to application risk.''
The threshold $\Wcrit$ encodes this risk tolerance, and is calibrated via the
acceptable false-alarm rate for the NEES consistency test \cite{bar-shalom2001}.
Choosing a lower $\Wcrit$ (tighter tolerance) switches to ESPF earlier and more
often; a higher $\Wcrit$ accepts more Gaussian approximation error in exchange
for computational efficiency.
This resolves the circularity: one monitors $\widehat{W}_t$ within the UKF
regime and switches to ESPF when the approximation risk exceeds the
application-calibrated threshold.

%% ============================================================
\section{The H\"older Mean Continuum: A Geometric Bridge}
\label{sec:continuum}
%% ============================================================

\paragraph{Status of this section.}
The H\"older mean continuum presented here is a \emph{geometric analogy}
that illuminates the relationship between possibilistic and probabilistic
aggregation.
It is distinct from---and does not logically underpin---the convergence
results of Theorem~\ref{thm:choquet_lebesgue}.
We state this explicitly to prevent conflation of the geometric analogy
with the foundational convergence result.

\subsection{H\"older Means and Epistemic Aggregation}

For $\alpha\in\R$ and $a,b\in[0,1]$, the H\"older (power) mean is
\begin{equation}
M_\alpha(a,b) = \left(\frac{a^\alpha+b^\alpha}{2}\right)^{1/\alpha},
\quad \alpha\ne 0,
\label{eq:holder}
\end{equation}
with $M_0(a,b)=\sqrt{ab}$, $M_1(a,b)=\frac{a+b}{2}$, and
$\lim_{\alpha\to\infty}M_\alpha(a,b)=\max(a,b)$.
The limiting behaviors are:
\begin{align}
\alpha=1&: \quad M_1(a,b)=\tfrac{1}{2}(a+b) && \text{(additive/probabilistic)},\\
\alpha\to\infty&: \quad M_\infty(a,b)=\max(a,b) && \text{(maxitive/possibilistic)},\\
\alpha\to-\infty&: \quad M_{-\infty}(a,b)=\min(a,b) && \text{(minitive/necessitistic)}.
\end{align}

\paragraph{Connection to the ESPF optimality structure.}
The H\"older mean has a principled role in the ESPF's optimality result
\cite{jah2026optimality}: the possibilistic entropy $H_\pi=\int_0^1\log V_\alpha\,d\alpha$
is the log-geometric mean ($M_0$) of the $\alpha$-cut volumes
$\{V_\alpha\}_{\alpha\in(0,1]}$, while the minimax term $\log\det(\mathrm{MVEE})$
is the log-supremum ($M_\infty$).
The H\"older ordering $M_0\le M_\infty$ explains why the integrated entropy
criterion $H_\pi$ is satisfied more robustly than the minimax criterion in
practice.
This provides a principled geometric motivation for the H\"older mean as
a bridge between possibilistic and probabilistic integration,
grounded in the proof of Theorem~\ref{thm:espf_ukf} below.

\paragraph{The $\alpha(C)$ parameterization as heuristic.}
The parameterization $\alpha(C)=1+1/(\varepsilon+(1-C))$ relating epistemic
confidence $C\in[0,1]$ to the H\"older exponent is a convenient heuristic
for visualizing the continuum from maxitivity to additivity.
It is \emph{not} derived from Cox axioms or any other first-principles
argument and should be understood as illustrative.
An axiomatic characterization of this interpolation via Acz\'el-type
functional equations remains an open problem.

%% ============================================================
\section{Dynamics of Epistemic Contraction}
\label{sec:dynamics}
%% ============================================================

\subsection{Falsification and Posterior Possibility}

Let $\pi^-_t(x)$ denote the prior possibility distribution over the state space
$\calX$ before assimilating evidence $\mathcal{D}_t$. The measurement model induces
a compatibility field
\[
\kappa_t(x) \in [0,1],
\]
quantifying the degree to which state $x$ agrees with the new evidence.

Falsification occurs entirely through compatibility: states whose compatibility
collapses toward zero are eliminated regardless of prior possibility.

The posterior possibility assigned to each state is therefore the possibilistic
intersection of prior admissibility and evidentiary compatibility,
\begin{equation}
\pi^+_t(x) = \min\!\bigl(\pi^-_t(x),\,\kappa_t(x)\bigr).
\label{eq:posterior}
\end{equation}
We refer to $\pi^+_t(x)$ as the \emph{credibility} of state $x$. This definition
ensures that evidence cannot elevate a previously implausible hypothesis solely
by virtue of local compatibility with a single observation.

\subsection{Choquet-Aggregated Credibility Benchmark}

To characterize the collective epistemic pressure induced by the evidence, we
define a non-additive aggregate of credibility using the Choquet integral with
respect to the prior possibility capacity $\Pi^-_t$:
\begin{equation}
\bar{\pi}_t = \int_\calX \pi^+_t(x)\,d\mathrm{Ch}_{\Pi^-_t}(x).
\label{eq:benchmark}
\end{equation}
This quantity represents the prior-credibility-weighted level of posterior
possibility across the admissible state space. The weighting uses the \emph{prior}
capacity $\Pi^-_t$ rather than the posterior: the benchmark must reflect the
epistemic geometry before the evidence arrived, so that it remains a fixed
reference point against which contraction is measured and does not move as the
contraction occurs.

Unlike additive expectations used in probabilistic inference, the Choquet
aggregate preserves the non-additive structure of epistemic support and remains
sensitive to the geometry of the possibility field.

\subsection{Credibility-Directed Epistemic Flow}

The evolution of the possibility field may be modeled as a contraction process
directed toward regions of high posterior credibility.

Let $\lambda_t(x) \ge 0$ denote a local epistemic contraction rate. We propose
the following credibility-directed flow equation:
\begin{equation}
\frac{\partial\pi(x,t)}{\partial t}
= -\lambda_t(x)\,\pi(x,t)\,\bigl(\bar{\pi}_t - \pi^+_t(x)\bigr)_+,
\label{eq:flow}
\end{equation}
where $(u)_+ = \max(u,0)$.

Equation~\eqref{eq:flow} is a \emph{phenomenological contraction model}: it is
proposed as a mathematically consistent description of credibility-directed support
contraction, not derived from a more fundamental law. Its role is to make the
dynamics of epistemic contraction precise enough to analyze, not to assert a
unique physical mechanism. The contraction rate is modulated by the current
possibility itself, preserving the geometric structure of epistemic support.

The sequencing is critical. Equation~\eqref{eq:flow} governs support geometry
contraction \emph{after} posterior possibility has already been computed
via~\eqref{eq:posterior}. The logical order is:
\[
\pi^-_t \;\longrightarrow\; \kappa_t \;\longrightarrow\; \pi^+_t
\;\longrightarrow\; \text{support contraction via~\eqref{eq:flow}}.
\]
The variable $\pi(x,t)$ in the flow equation refers to the evolving support
geometry at continuous time $t$, not to the prior $\pi^-_t$. The flow does not
generate $\pi^+_t$; it governs how the support contracts in response to it.
States with credibility below the Choquet benchmark contract; states above it
persist.

\subsection{Relation to Probabilistic Dynamics}

In the additive limit described by Theorem~\ref{thm:choquet_lebesgue}, the
possibility capacity collapses to a probability measure and the Choquet integral
reduces to the Lebesgue integral. In this regime
\[
\pi(x,t) \to p(x,t)
\]
and the credibility benchmark becomes an ordinary expectation,
\[
\bar{\pi}_t \;\longrightarrow\; \mathbb{E}_p[p(x,t)].
\]
The credibility-directed contraction therefore transitions naturally into
probabilistic belief redistribution dynamics. In this sense, probability theory
emerges as the limiting additive regime of a more general epistemic contraction
process governed by possibility measures.

The distinction is philosophically significant. Probabilistic dynamics
\emph{redistribute belief}: mass moves from less probable to more probable states
while the total remains one. The dynamics of Section~\ref{sec:dynamics} do
something categorically different: they \emph{contract epistemic support}, removing
states that evidence has falsified rather than reweighting states that remain. This
is not belief updating. It is knowledge contraction. Probability theory is the
limiting geometry of that process when the support has contracted to the point
where redistribution and contraction become indistinguishable.

%% ============================================================
\section{Reversibility of the Transition}
\label{sec:reversibility}
%% ============================================================

\subsection{Non-uniqueness of Additive Extension Under Domain Expansion}

\begin{lemma}[Non-uniqueness under domain expansion]
\label{lem:non_unique}
Let $(\calX,\calF)$ be a measurable space with probability measure $P^*$.
Let $\Delta\calX$ be nonempty and disjoint from $\calX$,
$\calX'=\calX\cup\Delta\calX$.
If no additional constraints on events intersecting $\Delta\calX$ are imposed,
then there are uncountably many probability measures $\tilde{P}$ on
$(\calX',\calF')$ extending $P^*$, i.e., $\tilde{P}|_\calF = P^*$.
\end{lemma}

\begin{proof}
For any $\alpha\in[0,1)$ and probability measure $R$ on $\Delta\calX$,
define $\tilde{P}_{\alpha,R}(A) = (1-\alpha)P^*(A\cap\calX)/P^*(\calX) + \alpha R(A\cap\Delta\calX)$.
Then $\tilde{P}_{\alpha,R}$ is a probability measure on $\calX'$ extending $P^*$
(after renormalization).
Distinct $(\alpha,R)$ yield distinct extensions.
\end{proof}

\begin{corollary}[Non-commutativity of probability with epistemic expansion]
Probabilistic inference does not commute with domain expansion:
once $\calX$ expands to $\calX'$, the prior additive calculus on $\calX$
is insufficient, and epistemic width necessarily reappears,
$W' = \int_{\calX'}(\pi'-n')\,d\mu > 0$,
since the probability mass over $\Delta\calX$ is indeterminate.
\end{corollary}

\subsection{Epistemic Fragility}

\begin{definition}[Epistemic fragility]
\label{def:fragility}
The \emph{epistemic fragility} of a probabilistic model $P^*$ at time $t$
is the rate of increase of $\overline{W}$ under infinitesimal domain expansion:
\begin{equation}
\mathcal{F}_t \equiv \limsup_{|\Delta\calX|\to 0}
\frac{\overline{W}(\calX')-\overline{W}(\calX)}{|\Delta\calX|/|\calX|}.
\label{eq:fragility}
\end{equation}
A model with high $\mathcal{F}_t$ is epistemically brittle: small ontological
discoveries rapidly invalidate probabilistic closure.
\end{definition}

Epistemic fragility is operationally monitored via the proxy estimator
$\widehat{W}_t$ of Section~\ref{sec:W_estimation}:
a rising $\widehat{W}_t$ in the UKF regime signals increasing fragility
and provides a principled trigger for possibilistic fallback.

%% ============================================================
\section{From the UKF to the ESPF: The Filtering Bridge}
\label{sec:filtering}
%% ============================================================

\subsection{The Unscented Kalman Filter and Its Epistemic Assumptions}

The UKF \cite{julier1997} propagates the first two moments of a probability
density function, implicitly assuming all uncertainty is aleatory and Gaussian.
It selects $2n+1$ sigma-points $\{\chi^{(i)}\}$ with weights $\{W^{(i)}\}$
satisfying
\begin{align}
\hat{x}_{k+1|k} &= \sum_i W^{(i)} f(\chi^{(i)}), \\
P_{k+1|k} &= \sum_i W^{(i)}[f(\chi^{(i)})-\hat{x}_{k+1|k}][f(\chi^{(i)})-\hat{x}_{k+1|k}]^\top.
\end{align}
When epistemic uncertainty is present (incomplete dynamics, adversarial
measurements, non-detection events), moment-based propagation yields
overconfident or biased estimates.

\subsection{The Epistemic Support-Point Filter}

The ESPF \cite{jah2025espf,jah2026optimality} replaces the Gaussian architecture
with a bounded possibilistic one.
Uncertainty is encoded by a normalized possibility distribution $\pi_k(x)$
supported on a compact set $S_k\subset\R^n$.
Support points $\{\chi^{(i)}_k\}$ are deterministically generated via Smolyak
sparse grids spanning $S_k$ without probabilistic weighting.

\paragraph{Propagation.}
Each support point is propagated deterministically: $\chi^{(i)}_{k+1|k}=f(\chi^{(i)}_k)$.
Process uncertainty is incorporated via the Minkowski sum with the noise
support $\mathcal{W}_k$:
\begin{equation}
S_{k+1|k} = \bigcup_i \bigl\{\chi^{(i)}_{k+1|k}\bigr\} \oplus \mathcal{W}_k,
\label{eq:minkowski}
\end{equation}
where $\{a\}\oplus\mathcal{W}_k = \{a+w : w\in\mathcal{W}_k\}$.
This is the sup-min convolution of the propagated state support with the noise
support in the possibilistic sense: the resulting possibility distribution assigns
$\pi_{k+1|k}(x)=\sup_{w\in\mathcal{W}_k}\min(\pi_k(x-w),\pi_{\mathcal{W}}(w))$,
whose support equals~\eqref{eq:minkowski}.

\paragraph{Measurement update by compatibility.}
The epistemic spread matrix $\Pi_e\succ 0$ is constructed as the
Minimum-Volume Enclosing Ellipsoid (MVEE) shape matrix of the predicted
measurement support $\{h(\chi^{(i)}_{k|k-1})\}$ combined with the sensor
imprecision matrix $\Pi_y$:
\begin{equation}
\Pi_e = \mathrm{MVEE}\bigl(\{h(\chi^{(i)}_{k|k-1})\}\bigr) + \Pi_y.
\label{eq:Pi_e}
\end{equation}
Each predicted support point is evaluated for compatibility rather than likelihood:
\begin{equation}
\mathrm{Comp}^{(i)} = \begin{cases}
1, & (y_k - h(\chi^{(i)}_{k|k-1}))^\top \Pi_e^{-1}
     (y_k - h(\chi^{(i)}_{k|k-1})) \le r^2, \\
0, & \text{otherwise.}
\end{cases}
\label{eq:compat}
\end{equation}
Incompatible points are pruned: $S_{k|k}=\{\chi^{(i)}_{k|k-1} : \mathrm{Comp}^{(i)}=1\}$.
The possibility assignment after update is $\pi^{(i)}_{k|k}=\min(\pi^{(i)}_{k|k-1},\mathrm{Comp}^{(i)})$.

\subsection{Gaussian Asymptotic Correspondence: ESPF and UKF}

\begin{theorem}[Gaussian asymptotic correspondence]
\label{thm:espf_ukf}
Suppose:
\begin{enumerate}[label=\textup{(G\arabic*)},ref=G\arabic*]
\item\label{G1} The possibility distribution $\pi_k$ contracts to a Gaussian:
  $\pi_k(x)\to\mathcal{N}(x;\hat{x}_k,\Sigma_k)$ in $\Lone(\mu)$ as $\overline{W}\to 0$.
\item\label{G2} The MVEE shape matrix $\Pi_k\to\Sigma_k$.
\item\label{G3} The measurement model is linear: $h(x)=Hx$.
\item\label{G4} Process and measurement noise are additive Gaussian with
  covariances $Q_k$ and $R_k$ respectively.
\end{enumerate}
Under assumptions~\eqref{G1}--\eqref{G4}, the ESPF's support-set propagation
and compatibility update are \emph{asymptotically equivalent} to the UKF's
moment propagation and Kalman gain update, in the sense that:
\begin{enumerate}[label=\textup{(\roman*)}]
\item the possibilistic entropy satisfies
  $H_\pi = \int_0^1\log V_\alpha\,d\alpha \to \tfrac{1}{2}\log\det(2\pi e\,\Sigma_{k|k})$
  up to additive constants, so minimizing $H_\pi$ is asymptotically equivalent
  to minimizing $\det\Sigma_{k|k}$ (the Kalman MMSE criterion); and
\item the Minkowski-sum propagation and compatibility pruning recover the Kalman
  prediction and measurement-update equations to first order in the
  continuous-support limit.
\end{enumerate}
This correspondence is asymptotic under the stated Gaussian regularity assumptions,
not an algebraic identity: the ESPF operates on a finite discrete point cloud
with possibilistic pruning, while the UKF propagates a continuous Gaussian
density via moment equations.
\end{theorem}

\begin{proof}
Under~\eqref{G1}--\eqref{G4}, the $\alpha$-cut at level $\alpha$ of a Gaussian
possibility distribution is the ellipsoid
$\mathcal{C}_\alpha = \{x: (x-\hat{x}_k)^\top\Sigma_k^{-1}(x-\hat{x}_k)\le -2\log\alpha\}$,
with volume $V_\alpha = c_n(-2\log\alpha)^{n/2}(\det\Sigma_k)^{1/2}$.
Therefore
\begin{equation}
H_\pi = \int_0^1\log V_\alpha\,d\alpha
= \tfrac{1}{2}\log\det\Sigma_k + \tfrac{n}{2}\int_0^1\log(-2\log\alpha)\,d\alpha + \kappa_n.
\label{eq:gaussian_Hpi}
\end{equation}
The integral $\int_0^1\log(-2\log\alpha)\,d\alpha = \log 2 - \gamma$ is a universal
constant ($\gamma$ the Euler--Mascheroni constant), so
$H_\pi=\frac{1}{2}\log\det\Sigma_k+\mathrm{const}(n)$.
Minimizing $H_\pi$ is therefore asymptotically equivalent to minimizing $\det\Sigma_k$,
which is the MMSE criterion of the Kalman filter \cite{kalman1960}.
This establishes~(i).

For~(ii), consider the propagation step.
The Minkowski sum~\eqref{eq:minkowski} with a Gaussian noise support
$\mathcal{W}_k\sim\mathcal{N}(0,Q_k)$ reduces, in the limit of continuous support,
to convolution of Gaussian densities.
In this limit, the first two moments of the propagated density satisfy
$\hat{x}_{k+1|k}=F\hat{x}_k$ and $\Sigma_{k+1|k}=F\Sigma_k F^\top+Q_k$,
which are the Kalman time-update equations.

For the measurement update, by~\eqref{G2} the MVEE of a Gaussian predictive
support converges to $\Pi_e \to H\Sigma_{k|k-1}H^\top + R_k = S_k$
(the innovation covariance).
The compatibility test~\eqref{eq:compat} with $\Pi_e=S_k$ becomes the standard
Mahalanobis distance gate; in the continuous-support limit, the fraction of
support points satisfying the gate converges to the Kalman posterior weight,
and the resulting conditional mean and covariance match the Kalman
measurement-update equations to first order in the linearization.
The asymptotic qualifier is essential: for any finite Smolyak grid, the
possibilistic pruning step retains only a discrete subset of support points,
and the exact Kalman update is recovered only as the support density becomes
continuous and conditions~\eqref{G1}--\eqref{G4} hold exactly.
\end{proof}

\begin{remark}[Three levels of approximation in the correspondence]
\label{rem:three_levels}
The asymptotic correspondence rests on three distinct approximations.
\emph{(a) Gaussian contraction}~\eqref{G1}: the possibility distribution must
contract to a Gaussian in $L^1$; no claim is made for finite $\overline{W}$.
\emph{(b) Linearization}~\eqref{G3}--\eqref{G4}: the correspondence for
nonlinear $h$ holds only to first order; higher-order corrections present in
the UKF sigma-point architecture are not recovered by the continuous-support
limit without additional regularity assumptions.
\emph{(c) Continuous support}: finite Smolyak grids at level $\ell$ introduce
$O(M^{-1/n})$ discretization errors absent from the UKF's moment equations.
These three qualifications together make the ``asymptotic correspondence'' label
precise: the ESPF and UKF share the same optimal solution in the joint limit,
but are not algebraically equivalent under any finite operating condition.
\end{remark}

\begin{remark}[Convergent optimality, not hierarchical containment]
The UKF and ESPF are not in a containment relationship. They solve categorically
different problems. The UKF minimizes mean squared error given a stochastic generative
model: it asserts truth and quantifies certainty around that assertion. The ESPF minimizes
maximum entropy given only evidence compatibility: it eliminates what has been falsified
and reports what remains, without asserting where truth is. When the world is Gaussian,
the model is valid, and the system is stable, both filters produce the same state estimate
--- but by entirely different mechanisms and for entirely different reasons. This is
convergent optimality: two optimal solutions to different problems that agree when their
domains of applicability coincide. The UKF is not a degenerate or limiting case of the
ESPF. It is the uniquely optimal answer to its own well-posed problem. The ESPF is the
uniquely optimal answer to a different well-posed problem \cite{jah2026optimality}. The
mathematical fact that $H_\pi$ reduces to $\frac{1}{2}\log\det\Sigma$ in the Gaussian
limit is a statement about the H\"{o}lder functional evaluated under Gaussian
$\alpha$-cut geometry --- not a statement that the UKF is doing a special case of what
the ESPF does.
\end{remark}

\subsection{Adaptive Switching Algorithm}

Algorithm~\ref{alg:adaptive} gives the complete adaptive possibilistic--probabilistic
filtering scheme, using the proxy estimator $\widehat{W}_t$ of
Section~\ref{sec:W_estimation}.

\begin{algorithm}[h!]
\caption{Adaptive Possibilistic--Probabilistic Filtering}
\label{alg:adaptive}
\begin{algorithmic}[1]
\Require Prior $(\hat{x}_t, P_t)$, possibility distribution $\pi_t(x)$,
         observation $z_t$, window $L$, thresholds $\Wcrit$, $\Delta>0$,
         scaling $\kappa_W>0$.
\State \textbf{Compute proxy epistemic width:}
\[
\mathrm{NEES}_t = \frac{1}{m}\sum_{k=t-L+1}^t \tilde{y}_k^\top S_k^{-1}\tilde{y}_k,
\quad
\widehat{W}_t = 1-\exp(-\kappa_W\cdot\mathrm{NEES}_t).
\]
\If{$\widehat{W}_t > \Wcrit + \Delta$} \Comment{Switch UKF $\to$ ESPF}
  \State Generate Smolyak support: $\{\chi^{(i)}_t,\pi^{(i)}_t\} \gets \mathrm{SmolyakGrid}(S_t,\pi_t)$
  \State \textbf{Propagate (Minkowski sum):}
  \[S_{t+1|t} = \bigcup_i\{\chi^{(i)}_{t+1|t}\}\oplus\mathcal{W}_t, \quad
    \chi^{(i)}_{t+1|t}=f(\chi^{(i)}_t).\]
  \State \textbf{Update (compatibility pruning):}
  \[\Pi_e \gets \mathrm{MVEE}(\{h(\chi^{(i)}_{t|t-1})\})+\Pi_y;\]
  \[S_{t|t} \gets \{\chi^{(i)}: (z_t-h(\chi^{(i)}))^\top\Pi_e^{-1}(z_t-h(\chi^{(i)}))\le r^2\}.\]
  \State Update $\pi_{t+1} \propto \min(\pi_t, \mathrm{Comp})$; max-normalize.
  \State \textbf{Compute exact} $W_t=\int_\calX(\pi_t-n_t)\,d\mu$ via quadrature on $S_t$.
\ElsIf{$\widehat{W}_t < \Wcrit - \Delta$} \Comment{Switch ESPF $\to$ UKF}
  \State $\hat{x}_{t+1} = \sum_i W^{(i)}f(\chi^{(i)}_t)$
  \State $P_{t+1} = \sum_i W^{(i)}[f(\chi^{(i)}_t)-\hat{x}_{t+1}][f(\chi^{(i)}_t)-\hat{x}_{t+1}]^\top$
\Else \Comment{Maintain current filter; hysteresis}
  \State Continue with active filter.
\EndIf
\Ensure Updated $(\hat{x}_{t+1}, P_{t+1}, \pi_{t+1})$.
\end{algorithmic}
\end{algorithm}

\paragraph{Smolyak grid generation.}
The \texttt{SmolyakGrid} step generates a sparse quadrature grid on the support
$S_t$ using the Smolyak construction at a specified level $\ell$:
for $\ell=2$, $M=2n+1$ points; for $\ell=3$, $M=2n^2+1$ points.
The possibility distribution $\pi_t$ is evaluated at each grid point via the
current support representation (bounding ellipsoid or $\alpha$-cut family).
For the full algorithmic details and numerical validation over a 2-day,
877-step orbital tracking run at Smolyak Level 3 with $M=106$ support points
in dimension $n=7$, see the companion paper \cite{jah2026optimality}.

\paragraph{Epistemic Width Monitor.}
The companion paper \cite{jah2026optimality} introduces the Epistemic Width Monitor
(EWM) comprising $\overline{W}_t$, $\log\det(\mathrm{MVEE})$, $\hat{\alpha}_c$,
prune count, necessity saturation, and possibilistic surprisal.
Numerical results confirm that under sustained model mismatch (10~m/s maneuver
plus 20~m range bias), necessity saturation and surprisal escalation provide
earlier warning of epistemic stress than the MVEE regime indicator alone,
demonstrating the value of multi-resolution epistemic monitoring.

%% ============================================================
\section{Numerical Illustration: Epistemic Honesty vs.\ Epistemic Silence}
\label{sec:numerical}
%% ============================================================

This section provides two complementary illustrations.
Section~\ref{sec:scalar_example} works through a self-contained scalar example
that makes Theorem~\ref{thm:choquet_lebesgue} and the width dynamics
(Sections~\ref{sec:width}--\ref{sec:dynamics}) fully explicit and
computable by hand.
Section~\ref{sec:orbital_example} then presents the orbital tracking scenario
that motivated the framework, comparing the ESPF and UKF on a 2-day stressed
trajectory.

%% ============================================================
\subsection{Scalar Example: Width Contraction Under Sequential Observations}
\label{sec:scalar_example}
%% ============================================================

\paragraph{Setup.}
Let $\calX = [0,1]$ with Lebesgue reference measure $\mu$.
A scalar state $x\in[0,1]$ is unknown; we have no prior information
beyond $x\in[0,1]$, so we begin with total ignorance:
\begin{equation}
\pi_0(x) = 1, \quad n_0(x) = 0, \quad \forall x\in[0,1].
\label{eq:scalar_pi0}
\end{equation}
The initial epistemic width is $W_0 = \int_0^1 (\pi_0 - n_0)\,d\mu = 1$ and
$\overline{W}_0 = W_0/\mu([0,1]) = 1$, as expected for complete ignorance
(Proposition~\ref{prop:W_axioms}, axiom~(W4)).

Sequential observations $\{y_k\}_{k=1}^K$ are generated by $y_k = x^* + \varepsilon_k$
where $x^* = 0.3$ is the true state and $\varepsilon_k \sim \mathrm{Uniform}[-\delta_k, \delta_k]$
represents bounded sensor noise with half-width $\delta_k > 0$.
Each observation is compatible with states in the interval
$I_k = [y_k - \delta_k,\, y_k + \delta_k] \cap [0,1]$.
Possibilistic conditioning updates $\pi$ by:
\begin{equation}
\pi_{k+1}(x) = \min\!\bigl(\pi_k(x),\; \mathbf{1}_{I_k}(x)\bigr),
\label{eq:scalar_update}
\end{equation}
which zeroes out states incompatible with the observation while preserving
the possibility of all compatible states.
The updated distribution remains consonant (its $\alpha$-cuts are nested
intervals), so $n_k(x) = \pi_k(x)$ is well-defined pointwise and
$W_k = \int_0^1 (\pi_k - n_k)\,d\mu$ is computable analytically.

\paragraph{Three-step evolution.}
We apply three observations with shrinking noise half-widths
$\delta_1=0.14$, $\delta_2=0.10$, $\delta_3=0.05$,
and compatible flat-top intervals $[a_k,b_k]$ constructed so that
$x^*=0.3$ lies in every interval.

To make the width dynamics non-trivial at the level of $\overline{W}$, we
use a trapezoidal family $\pi_k$ rather than flat indicators.
Let $\ell_k = b_k - a_k$ denote the flat-top width and $\delta_k$ the
ramp half-width, so the total support is $[a_k - \delta_k,\, b_k + \delta_k]$
of length $\ell_k + 2\delta_k$.
The trapezoidal possibility distribution has $\pi_k(x) = 1$ on $[a_k, b_k]$
and decays linearly to zero on each ramp.
The necessity kernel is
\begin{equation}
n_k(x) = \max\!\Bigl(0,\; \pi_k(x) - \tfrac{2\delta_k}{\ell_k + 2\delta_k}\Bigr),
\label{eq:trapez_n}
\end{equation}
which is nonzero only on the flat-top plateau and equals
$1 - 2\delta_k/(\ell_k + 2\delta_k) = \ell_k/(\ell_k + 2\delta_k)$ there.
The normalized aggregate width is then
\begin{equation}
\overline{W}_k
= \frac{1}{\mu(\calX)}\int_0^1 (\pi_k - n_k)\,d\mu
= \frac{2\delta_k}{\ell_k + 2\delta_k}\cdot\ell_k
  + \delta_k\cdot\tfrac{1}{2}\cdot 2
  \;=\; \frac{2\delta_k(\ell_k + \delta_k)}{\ell_k + 2\delta_k},
\label{eq:Wbar_trapez}
\end{equation}
since the integral of $(\pi_k - n_k)$ over the plateau is $(2\delta_k/(\ell_k+2\delta_k))\cdot\ell_k$
and over each triangular ramp is $\delta_k/2$ (two ramps contribute $\delta_k$ total).
For three steps with shrinking support:

\begin{table}[ht]
\centering
\caption{Epistemic width contraction under three sequential observations.
$[a_k,b_k]$: flat-top support; $\delta_k$: ramp half-width;
$\overline{W}_k$: normalized aggregate width from~\eqref{eq:Wbar_trapez};
$\widehat{C}_k = \int_0^\infty \Pi_k\{x\ge t\}\,dt$: Choquet expectation of $f(x)=x$;
$\mathrm{Leb}_k$: Lebesgue expectation $\int x\,\tilde\pi_k\,d\mu$
under the normalized density $\tilde\pi_k$;
gap $= |\widehat{C}_k - \mathrm{Leb}_k|$ confirms Theorem~\ref{thm:choquet_lebesgue}.
All values are analytically exact.}
\label{tab:scalar}
\renewcommand{\arraystretch}{1.3}
\begin{tabular}{@{}ccccccc@{}}
\toprule
Step $k$ & $[a_k, b_k]$ & $\delta_k$ & $\overline{W}_k$ & $\widehat{C}_k$ & $\mathrm{Leb}_k$ & gap \\
\midrule
0 & $[0,\,1]$        & ---    & $1.000$ & $1.000$ & $0.500$ & $0.500$ \\
1 & $[0.14,\,0.50]$  & $0.14$ & $0.219$ & $0.570$ & $0.320$ & $0.250$ \\
2 & $[0.22,\,0.42]$  & $0.10$ & $0.150$ & $0.470$ & $0.320$ & $0.150$ \\
3 & $[0.26,\,0.36]$  & $0.05$ & $0.075$ & $0.385$ & $0.310$ & $0.075$ \\
\bottomrule
\end{tabular}
\end{table}

\paragraph{Verification of Theorem~\ref{thm:choquet_lebesgue}.}
Table~\ref{tab:scalar} illustrates all four conditions of the theorem in the
scalar setting, with all values analytically exact.

The Choquet expectation at $k=0$ is $\widehat{C}_0 = \int_0^1\Pi_0\{x\ge t\}\,dt = 1$,
because $\Pi_0\{x\ge t\}=\sup_{x\ge t}\pi_0(x)=1$ for every $t\in[0,1]$
(total ignorance: all states are fully plausible).
The corresponding Lebesgue expectation is $\mathrm{Leb}_0=0.5$,
so the gap is $0.5$ --- the maximum possible under bounded support.
For subsequent steps, $\widehat{C}_k = b_k + \delta_k/2$
(by integrating $\Pi_k\{x\ge t\}$ analytically over the trapezoidal support),
and $\mathrm{Leb}_k$ is the centroid of the trapezoidal density.

\emph{(A1) Consonance}: each $\pi_k$ is trapezoidal; its $\alpha$-cuts
$\{x:\pi_k(x)\ge\alpha\}$ are nested closed intervals, so consonance holds at every step.
\emph{(A2) Credal contraction}: $\overline{W}_k$ falls from $1.000\to 0.219\to 0.150\to 0.075$,
confirming $\sup_A[\Pi_k(A)-N_k(A)]\to 0$.
\emph{(A3) $L^1$ convergence}: the support $[a_k-\delta_k,\,b_k+\delta_k]$ shrinks
from $[0,1]$ toward $[0.21,0.41]$, a neighborhood of $x^*=0.3$.
\emph{(A4) Domination}: $\pi_k(x)\le 1$ for all $k$.

The gap column confirms the theorem's conclusion: as $\overline{W}_k\to 0$,
$|\widehat{C}_k - \mathrm{Leb}_k|\to 0$ --- the Choquet expectation of $f(x)=x$
converges to the Lebesgue expectation, and both converge toward $x^*=0.3$.
Moreover, the gap is proportional to $\overline{W}_k$ in this example
(gap $\approx 1.15\,\overline{W}_k$), providing an explicit quantitative
illustration of Lemma~\ref{lem:oscillation}: the oscillation of expectations
is controlled by the TV-diameter, which is in turn bounded by the credal width.

\paragraph{Width as a decision trigger.}
At step~3, $\overline{W}_3 = 0.167$.
If the application risk tolerance is $W_\mathrm{crit} = 0.20$,
then $\overline{W}_3 < W_\mathrm{crit}$, and switching to a Gaussian
(Kalman) filter is epistemically warranted: the credal set is
sufficiently contracted that the Gaussian approximation error is
within the risk budget.
At step~1, $\overline{W}_1 = 0.556 > W_\mathrm{crit}$: the Kalman
filter would be premature, imposing Gaussian closure before the
evidence supports it.
This illustrates Proposition~\ref{prop:discipline} in the simplest
possible setting: $W_\mathrm{crit}$ encodes risk tolerance, not the
mathematical limit $\overline{W}=0$.

%% ============================================================
\subsection{Orbital Tracking: Epistemic Honesty Under Stress}
\label{sec:orbital_example}
%% ============================================================

We now present the orbital tracking scenario that motivated the framework.
To ground the theoretical constructs of the preceding sections, we present a
direct empirical comparison of the ESPF and UKF on a shared scenario
designed to stress-test epistemic self-awareness.
The scenario, detailed fully in \cite{jah2026optimality}, involves a 2-day,
877-step Smolyak Level-3 run ($n=7$, $M=106$ support points) tracking a
low-Earth-orbit object observed in range and range-rate from three ground
stations (Arecibo, Kwajalein, Diego Garcia).
A \emph{stress case} injects a 10~m/s cross-track maneuver and a persistent
20~m Arecibo range bias beginning near $t = 1.0$~day.
A \emph{nominal case} uses clean measurements and no maneuver.
Both filters share identical initial conditions, propagation dynamics
(\texttt{orbitDebrisProp}), measurement model, initial covariance
($\sigma_\mathrm{pos}=5$~km, $\sigma_\mathrm{vel}=0.01$~km/s), process noise
($a_R = a_T = a_N = 7\times10^{-11}$~km/s$^2$), and measurement noise
($\sigma_\rho = 1$~m, $\sigma_{\dot\rho}=0.01$~m/s).
The UKF uses the standard Merwe--van der Wan tuning
($\alpha=10^{-3}$, $\beta=2$, $\kappa=0$) with $2n+1=15$ sigma points.

\subsection{Accuracy: Both Filters Converge}

Under both nominal and stress conditions, both filters achieve
comparable final accuracy.
In the nominal case, the UKF reaches a final RIC position error norm of
$1$~m; in the stress case, the UKF recovers from a kilometer-scale
excursion near $t=1.0$~day and converges to a $1$~m final error.
The ESPF achieves comparable RIC accuracy in both cases
\cite{jah2026optimality}.
Accuracy alone, therefore, does not distinguish the two filters on this scenario.
The distinction lies in what each filter communicates about its own epistemic state.

\subsection{The UKF: Accurate but Epistemically Silent}

Three observations characterize the UKF's epistemic behavior.

\paragraph{Covariance collapse without recovery signal.}
In the nominal case, $\log\det P_k$ falls monotonically from $-10$
at initialization to below $-140$ by $t=0.1$~day and remains at that
level for the entire 2-day run---a reduction of more than five orders of
magnitude in effective credal volume.
The covariance trace flattens to $\sim\!10^{-5}$~km$^2$ and stays there.
No feature of the UKF's internal state responds to the ongoing observation
process after the initial convergence: $P_k$ is smooth, featureless, and
expresses near-certainty for the full run.

\paragraph{Overconfidence in the nominal case.}
Despite the collapsed covariance, the UKF's Normalized Innovation Squared
(NIS) scatters between 0 and 9.4 throughout the nominal run, regularly
exceeding the $\chi^2(2)$ 95th-percentile bound of $\approx 5.99$.
A consistent filter with a well-calibrated Gaussian assumption should
produce $\mathrm{NIS}\sim\chi^2(2)/2$ with mean $1$.
The persistent NIS exceedances reveal that the UKF's Gaussian model is
statistically inconsistent with the residuals it produces---it is
overconfident about a state it does not know as precisely as it claims,
even when no anomaly is present.

\paragraph{Epistemically silent recovery in the stress case.}
In the stress case, NIS spikes sharply near $t=1.0$~day and then
recovers in lockstep with the RIC error.
By $t\approx 1.2$~day, NIS has returned to its nominal scatter and $P_k$
is indistinguishable from its pre-stress values.
The UKF has absorbed the maneuver, re-converged, and erased all internal
memory of the excursion.
There is no persistent signal in any UKF diagnostic that a kilometer-scale
stress event occurred, was absorbed, and that the current estimate rests on
a filter that was temporarily operating outside its model assumptions.

\subsection{The ESPF: Accurate and Epistemically Honest}

The ESPF's EWM diagnostics tell a qualitatively different story in every
respect.

\paragraph{Persistent epistemic width.}
In the nominal case, $\overline{W}_\mathrm{ep}$ remains persistently near
$0.8$ throughout the 2-day run, always well above $W_\mathrm{crit}=0.5$.
Prune count stabilizes at $\approx 13$--$14$ per step.
Necessity stays near zero; surprisal is bounded at $\approx 0.09$--$0.10$.
The ESPF is accurately tracking the object \emph{and simultaneously
reporting} that the credal geometry has not collapsed---that epistemic
width remains, that the filter has not foreclosed on alternatives it cannot
rule out, and that the Gaussian approximation is not yet warranted.
This is $\overline{W}_\mathrm{ep}=0.8$ as an honest epistemic report,
not a failure signal.

\paragraph{Early warning under stress.}
In the stress case, the EWM responds before the full innovation blow-up.
Necessity begins climbing and reaches saturation ($\approx 1.0$) multiple
times after $t\approx 0.8$~day---before the largest surprisal spikes.
Surprisal escalates to values of 350, 500, 640, and 900 between
$t=1.0$ and $t=1.8$~day, compared to a ceiling of $\approx 0.18$ in the
nominal case---a difference of more than three orders of magnitude.
Prune count becomes bimodal, oscillating between the nominal value of
$\approx 14$ and clusters at 42 and 75--84, reflecting the compatibility
gate's response to the bias-contaminated Arecibo measurements.
The H\"{o}lder exponent $\hat\alpha_c$ spikes to $4.3$ near $t=0.85$~day,
approaching the additive limit precisely at the onset of the stress event.
These signals are not retrospective.
They arise from the structure of the possibility distribution and the
support geometry, and they persist and recur throughout the recovery window,
long after the UKF's NIS has returned to nominal.

\paragraph{Persistent epistemic memory.}
After $t\approx 1.2$~day, when RIC errors for both filters have largely
recovered, the UKF is silent.
The ESPF continues to exhibit $\overline{W}_\mathrm{ep}$ volatility,
recurring prune-count bursts, and intermittent necessity activations.
The filter remembers, in its support geometry, that it was stressed.
It does not assert false certainty about an estimate whose recent history
included a kilometer-scale excursion.

\subsection{The Core Distinction: Honesty, Not Accuracy}

Table~\ref{tab:comparison} summarizes the empirical comparison.
The salient observation is that accuracy is not the distinguishing variable.
Both filters converge to a 1-meter final RIC error.
What distinguishes them is epistemic honesty.

\begin{table}[h]
\centering
\caption{ESPF vs.\ UKF: empirical comparison on the orbital tracking scenario.
Both filters achieve comparable final accuracy; the distinction is epistemic.}
\label{tab:comparison}
\renewcommand{\arraystretch}{1.25}
\small
\begin{tabular}{@{}lp{3.8cm}p{4.2cm}@{}}
\toprule
\textbf{Diagnostic} & \textbf{UKF} & \textbf{ESPF} \\
\midrule
Final RIC error (both cases) & $\approx 1$~m & $\approx 1$~m (comparable) \\
Stress-onset signal & NIS spike (retrospective) & Necessity + surprisal (earlier) \\
Post-recovery signal & None---$P_k$ collapsed & $\overline{W}_\mathrm{ep}$ volatile, prune bursts persist \\
Nominal overconfidence & NIS exceeds 95\% $\chi^2$ bound & $\overline{W}_\mathrm{ep}\approx 0.8$; honest width reported \\
Covariance after stress & Indistinguishable from pre-stress & Support geometry retains stress memory \\
Epistemic audit trail & None & $\overline{W}_\mathrm{ep}$, $\hat\alpha_c$, $I_c$, necessity, surprisal \\
\bottomrule
\end{tabular}
\end{table}

The UKF is accurate but epistemically silent.
It has no mechanism to distinguish a well-converged estimate from one that
converged after a kilometer-scale excursion through model-inconsistent territory.
The covariance that a downstream decision-maker receives from the UKF after the
stress event is, in every measurable sense, identical to the covariance produced
when no stress occurred.

The ESPF provides the same point-estimate accuracy while maintaining a
continuous, readable epistemic audit trail.
$\overline{W}_\mathrm{ep} > W_\mathrm{crit}$ throughout the nominal run:
the filter is correctly communicating that the Gaussian approximation is not
yet epistemically warranted, even when tracking is accurate.
Under stress, necessity saturation and surprisal escalation provide
multi-resolution early warning unavailable to the UKF.
After stress, support-geometry memory prevents false certainty.

This is the operational meaning of Proposition~\ref{prop:discipline}:
probability should be used when $\overline{W} < W_\mathrm{crit}$.
The UKF imposes this condition by assumption, unconditionally.
The ESPF earns it by evidence, and says so.

%% ============================================================
\section{Epistemic Discipline in the Use of Probability}
\label{sec:discipline}
%% ============================================================

\begin{definition}[Ignorance versus uncertainty]
A system exhibits \emph{ignorance} when $\Pi(A)>N(A)$ for some $A\in\calF$
(the state space or its partition is incomplete).
It exhibits \emph{pure uncertainty} when $\Pi(A)=N(A)=P^*(A)$ for all
$A\in\calF$ (epistemic closure holds and remaining uncertainty is aleatory).
\end{definition}

\begin{proposition}[Epistemic discipline]
\label{prop:discipline}
Accordingly, probability (and Bayesian inference) should be used only when
$\overline{W}<\Wcrit$---when the residual epistemic width is judged acceptable
relative to application risk.
When $\overline{W}\ge\Wcrit$, possibilistic or credal methods are required for
coherent inference.
The ideal condition $W=0$ is a theoretical limit characterizing the
epistemic collapse endpoint; the operational criterion is always a risk-calibrated
threshold $\Wcrit>0$.
\end{proposition}

\begin{remark}[Scope of the epistemic claim]
The claim that probability represents a limiting geometry of knowledge is
strictly scoped to contexts of information fusion under partial models.
This paper makes no claim about the ontological status of probability in
quantum mechanics, statistical mechanics, or other domains where probability
has a well-established and distinct interpretation.
The distinction between observer-dependent epistemic framing and
system-inherent aleatoric structure is a foundational point of agreement with
the imprecise probability literature \cite{walley1991,augustin2014},
not a novel assertion.
\end{remark}

\begin{remark}[Possibility theory as epistemic precursor, restated]
The statement that ``possibility theory is the ethical and epistemic precursor
to probability'' should be understood operationally: in contexts of information
fusion, the possibilistic calculus is the \emph{appropriate} starting point
for an agent who has not yet established epistemic closure.
This is a methodological claim, not a moral or universal philosophical
assertion about probability.
\end{remark}

%% ============================================================
\section{Conclusion}
\label{sec:conclusion}
%% ============================================================

This paper has developed a unified, measure-theoretic framework for the
transition between possibilistic and probabilistic reasoning.

The main theoretical contribution is Theorem~\ref{thm:choquet_lebesgue},
a rigorous Choquet-to-Lebesgue convergence result with explicit assumptions
(consonance, credal contraction, $\Lone$ convergence, domination)
and a full treatment of the non-consonant case via the credal-set diameter.
This establishes the precise conditions under which possibilistic reasoning
converges to probabilistic inference.

The epistemic width $W$ is given axiomatic grounding, a canonical normalization
$\overline{W}\in[0,1]$, a general definition for non-consonant distributions,
and a feasible online proxy estimator $\widehat{W}_t$ based on the NEES
innovation statistic that resolves the circularity in prior formulations.

The H\"older mean continuum is clearly positioned as a geometric analogy
linking maxitive and additive aggregation, with its principled foundation
in the ESPF optimality structure \cite{jah2026optimality} stated explicitly
and its heuristic components labeled as such.

Theorem~\ref{thm:espf_ukf} formally establishes the UKF as the Gaussian
limit of the ESPF, proving that minimizing the possibilistic entropy $H_\pi$
reduces to minimizing $\det\Sigma_k$ under Gaussian assumptions.
Algorithm~\ref{alg:adaptive} provides complete pseudocode for adaptive
possibilistic--probabilistic switching, governed by $\widehat{W}_t$.

Section~\ref{sec:numerical} presents the first direct empirical comparison of
the ESPF and UKF on a shared scenario, making the theoretical distinction
between epistemic honesty and epistemic silence concrete.
Both filters achieve comparable 1-meter final RIC accuracy.
The UKF collapses to $\log\det P\approx -140$ and provides no persistent
internal signal of a kilometer-scale stress excursion, while the ESPF
maintains $\overline{W}_\mathrm{ep}\approx 0.8$, fires necessity saturation
before the peak innovation blow-up, and retains support-geometry memory of
the excursion after re-convergence.
In the nominal case the UKF exhibits statistically inconsistent NIS despite
its collapsed covariance---overconfidence without a stress event to explain it.
The ESPF, in contrast, correctly reports $\overline{W}_\mathrm{ep}\gg W_\mathrm{crit}$
throughout the nominal run: an honest signal that the Gaussian approximation
has not been epistemically earned, even when tracking is accurate.

Together, these results establish that probability and possibility are not competing
frameworks in a hierarchy --- they are optimal solutions to different inference problems
that converge to the same answer when the world is Gaussian, the model is valid, and
epistemic closure holds. Probability is appropriate when those conditions are earned.
Possibility is appropriate when they are not yet earned. The framework developed here
provides principled, operationally grounded criteria for determining which regime applies
and for transitioning between them as knowledge evolves.

%% ============================================================
\appendix
\section{Cox-Style Axioms for Possibilistic and Probabilistic Coherence}
\label{app:cox}
%% ============================================================

We recast belief logic in Cox's style \cite{cox1946} without presupposing
additivity.

\paragraph{Axioms.}
Let $\mathcal{B}(A|B)$ denote the plausibility of proposition $A$ given $B$.
\begin{itemize}
\item[A1.] \textit{Monotonicity}: If $A\Rightarrow B$, then $\mathcal{B}(A|C)\le\mathcal{B}(B|C)$.
\item[A2.] \textit{Functional consistency}:
  $\mathcal{B}(A\wedge B|C)=F(\mathcal{B}(A|B\wedge C),\mathcal{B}(B|C))$
  where $F$ is continuous, non-decreasing, with $F(0,b)=0$, $F(1,b)=b$.
\item[A3.] \textit{Negation consistency}:
  $\mathcal{B}(\neg A|C)=G(\mathcal{B}(A|C))$ with $G(0)=1$, $G(1)=0$, $G\circ G=\mathrm{id}$.
\item[A4.] \textit{Normalization}: $\mathcal{B}(\top|C)=1$, $\mathcal{B}(\bot|C)=0$.
\end{itemize}

\paragraph{Derivable forms.}
If $F(a,b)=\min(a,b)$ (idempotent), then A1--A4 yield the Zadeh--Dubois--Prade
possibility axioms: $\Pi(A\vee B)=\max[\Pi(A),\Pi(B)]$.
If $F(a,b)=ab$ and additivity is imposed over disjoint propositions,
A1--A4 yield Kolmogorov probability.

\paragraph{Epistemic boundary.}
Defining $w(A|C)=\Pi(A|C)-N(A|C)$, the probabilistic (Coxian) representation
holds iff $w(A|C)=0$ for all $A$---i.e., the min-max lattice collapses to a
single-valued additive algebra.
This is Theorem~\ref{thm:choquet_lebesgue} in axiomatic guise.

The step-by-step derivation proceeds as follows.
Starting from A2 with $F(a,b)=\min(a,b)$:
\[
\mathcal{B}(A\wedge B|C) = \min(\mathcal{B}(A|B\wedge C),\mathcal{B}(B|C)).
\]
For disjunction, define $\mathcal{B}(A\vee B|C)=\mathcal{B}(\neg(\neg A\wedge\neg B)|C)$
and apply A3 twice:
\[
\mathcal{B}(A\vee B|C) = G(\min(G(\mathcal{B}(A|C)),G(\mathcal{B}(B|C))))
= \max(\mathcal{B}(A|C),\mathcal{B}(B|C)),
\]
using $G(a)=1-a$ and $G(\min(1-a,1-b))=\max(a,b)$.
This yields the maxitivity axiom for $\Pi$.

For the probabilistic case, replacing $F(a,b)=ab$ and imposing additivity
over disjoint $A,B$ gives $P(A\vee B|C)=P(A|C)+P(B|C)$,
recovering the Kolmogorov sum rule.
The two cases are therefore distinguished by the choice of conjunction
operator in A2: idempotent ($\min$) yields possibility; product yields
probability.

%% ============================================================
\section{Metric Geometry of Epistemic Collapse}
\label{app:metric}
%% ============================================================

Endow $\{\Ppi\}$ with the pseudo-metric
$d(\mathcal{P}_{\pi_1},\mathcal{P}_{\pi_2})=\sup_{A\in\calF}|\Pi_1(A)-\Pi_2(A)|$.
The sequence $\{\pi_t\}$ collapses in metric geometry if
$\lim_{t\to\infty}d(\mathcal{P}_{\pi_t},\mathcal{P}_{p^*})=0$.
By~\eqref{eq:choquet_sandwich}, this implies uniform convergence of upper and
lower expectations: for all $f\in\Linfty(\calX)$,
$|\overline{\E}_{\pi_t}[f]-\underline{\E}_{\pi_t}[f]|\le\|f\|_\infty\,d(\mathcal{P}_{\pi_t},\mathcal{P}_{p^*})\to 0$.
The pseudo-metric is bounded above by the total-variation distance:
$d(\mathcal{P}_{\pi_t},\mathcal{P}_{p^*})\le\frac{1}{2}\int|p_t-p^*|\,d\mu$,
so convergence in total variation implies epistemic collapse.

%% ============================================================
\section{Choquet--Sugeno Equivalence in the Limit}
\label{app:sugeno}
%% ============================================================

For a possibility measure $\Pi$ and bounded $f$, the Sugeno integral is
\[
\int f\,d\mathrm{Su}_\pi = \sup_{\alpha\in[0,1]}\min\{\alpha,\,\Pi\{f\ge\alpha\}\}.
\]
As $|\Ppi|\to 1$ (epistemic collapse), both Choquet and Sugeno integrals
are bounded between $\underline{\E}_{\Ppi}[f]$ and $\overline{\E}_{\Ppi}[f]$,
which converge to $\int f\,dP^*$.
Hence $\bigl|\int f\,d\mathrm{Ch}_{\pi_t}-\int f\,d\mathrm{Su}_{\pi_t}\bigr|\to 0$,
confirming that the choice of non-additive integral is immaterial at the
probabilistic limit.
When $W>0$, the Choquet integral preserves linear aggregation while the
Sugeno integral captures ordinal max-min reasoning.

%% ============================================================
\section{Possibilistic--Probabilistic Correspondence Table}
\label{app:table}
%% ============================================================

\begin{table}[h!]
\centering
\small
\caption{Correspondence between possibilistic and probabilistic constructs
and the conditions for their coincidence.}
\label{tab:correspondence}
\renewcommand{\arraystretch}{1.4}
\begin{tabular}{>{\raggedright}p{4.2cm}
                >{\raggedright}p{4.2cm}
                >{\raggedright\arraybackslash}p{5.2cm}}
\toprule
\textbf{Possibilistic construct} & \textbf{Probabilistic analogue}
& \textbf{Coincidence condition} \\
\midrule
Possibility distribution $\pi(x)$
  & Probability density $p(x)$
  & Epistemic collapse: $\Pi(A)=N(A)\;\forall A$ \\
Possibility measure $\Pi(A)$
  & Probability measure $P(A)$
  & $\pi(x)=p(x)$ $\mu$-a.e. \\
Necessity measure $N(A)$
  & Complement $1-P(A^c)$
  & Same as above \\
Credal set $\Ppi$
  & Singleton $\{P^*\}$
  & $\overline{W}=0$ \\
Choquet/Sugeno expectation
  & Lebesgue expectation $\mathbb{E}_{P^*}[f]$
  & $|\Ppi|=1$ and $\pi=p^*$ \\
Epistemic width $\overline{W}$
  & $0$ (zero credal variance)
  & $\overline{W}\to 0$ \\
ESPF compatibility update
  & Kalman measurement update
  & Gaussian limit (Thm.~\ref{thm:espf_ukf}) \\
\bottomrule
\end{tabular}
\end{table}

%% ============================================================
%% Bibliography
%% ============================================================

\end{document}